\documentclass[a4paper,11pt]{article}
\usepackage[utf8]{inputenc}
\usepackage[T1]{fontenc}
\usepackage[margin=1in]{geometry}
\usepackage[colorlinks,linkcolor=blue,citecolor=blue]{hyperref}

\usepackage{amsfonts,amsmath,amsthm,amssymb}
\usepackage{mathtools}
\usepackage{mathabx}
\usepackage[cal=boondoxo]{mathalfa}

\usepackage{float,framed}
\usepackage[noend]{algorithmic}

\usepackage[noblocks]{authblk}
\usepackage[numbers]{natbib}
\usepackage{enumitem}
\usepackage{bm}
\usepackage{xspace}

\usepackage{nicefrac}
\usepackage{pgfplots}

\floatstyle{plain}
\newfloat{algorithm}{h}{lop}
\floatname{algorithm}{Algorithm}

\usepackage{enumitem}
\usepackage{bm}
\usepackage{xspace}
\usepackage{nicefrac}

\usepackage[capitalise,nameinlink]{cleveref}
\usepackage{xcolor}
\usepackage{dsfont}

\newcommand{\ignore}[1]{}

\newcommand{\wrapalgo}[2][0.9\linewidth]
{%
\begin{center}\setlength{\fboxsep}{5pt}\fbox{\begin{minipage}{#1}
#2
\end{minipage}}\end{center}
\vspace{-0.25cm}
}

\theoremstyle{definition}
\newtheorem{theorem}{Theorem}				

\newtheorem{lemma}[theorem]{Lemma}

\newtheorem*{theorem*}{Theorem}
\newtheorem*{lemma*}{Lemma}
\newtheorem*{corollary*}{Corollary}
\newtheorem*{proposition*}{Proposition}
\newtheorem*{claim*}{Claim}
\newtheorem*{fact*}{Fact}
\newtheorem*{observation*}{Observation}

\theoremstyle{definition}

\newtheorem*{definition*}{Definition}
\newtheorem*{remark*}{Remark}
\newtheorem*{example*}{Example}
\newtheorem*{question*}{Question}

 \theoremstyle{plain}
\newtheorem*{theoremaux}{\theoremauxref}
\gdef\theoremauxref{1}

\DeclareMathAlphabet{\mathbfsf}{\encodingdefault}{\sfdefault}{bx}{n}

\DeclareMathOperator*{\argmin}{arg\!\min}

\DeclareMathOperator*{\trace}{Tr}
\DeclareMathOperator*{\diag}{diag}

\newcommand{\lr}[1]{\mathopen{}\left(#1\right)}
\newcommand{\Lr}[1]{\mathopen{}\big(#1\big)}
\newcommand{\LR}[1]{\mathopen{}\Big(#1\Big)}

\newcommand{\Lrbra}[1]{\mathopen{}\big[#1\big]}
\newcommand{\LRbra}[1]{\mathopen{}\Big[#1\Big]}

\newcommand{\norm}[1]{\|#1\|}

\newcommand{\lrset}[1]{\mathopen{}\left\{#1\right\}}
\newcommand{\Lrset}[1]{\mathopen{}\big\{#1\big\}}

\newcommand{\wt}[1]{\smash{\widetilde{#1}}}
\newcommand{\wh}[1]{\smash{\widehat{#1}}}
\renewcommand{\O}{O}
\newcommand{\ind}[1]{\mathbb{I}\lrset{#1}}

\newcommand{\tr}{^{\mkern-1.5mu\scriptstyle\mathsf{T}}}

\newcommand{\st}{\star}

\newcommand{\reals}{\mathbb{R}}
\newcommand{\eps}{\epsilon}

\newcommand{\thalf}{\tfrac{1}{2}}

\newcommand{\eqdef}{\stackrel{\text{def}}{=}}

\renewcommand{\leq}{~\le~}
\renewcommand{\geq}{~\ge~}

\let\oldtfrac\tfrac
\renewcommand{\tfrac}[2]{\smash{\oldtfrac{#1}{#2}}}

\let\nablaold\nabla
\renewcommand{\nabla}{\nablaold\mkern-1mu}

\makeatletter
\newcommand*\samethanks[1][\value{footnote}]{\footnotemark[#1]}
\renewcommand*{\@fnsymbol}[1]{\ensuremath{
  \ifcase#1 \or \natural \or \star \else\@ctrerr\fi}}
\makeatother

\title{Shampoo: Preconditioned Stochastic Tensor Optimization}

\author{%
Vineet Gupta%
\thanks{Google Brain. Email: \texttt{\{vineet,tkoren\}@google\!.\!com}}\qquad 
Tomer Koren\samethanks\qquad 
Yoram Singer%
\thanks{Princeton University and Google Brain. Email: \texttt{y\!.\!s@cs\!.\!princeton\!.\!edu}}
}

\begin{document}
\maketitle


\newcommand{\bbS}{\mathbb{S}}
\newcommand{\cW}{\mathcal{W}}
\newcommand{\cH}{\mathcal{H}}
\newcommand{\cR}{\mathcal{R}}
\newcommand{\barG}{\overline{G}}

\newcommand{\shalf}{{\smash{\nicefrac{1}{2}}}}
\newcommand{\squar}{{\smash{\nicefrac{1}{4}}}}

\newcommand{\kron}{\otimes}
\newcommand{\bigkron}{\bigotimes}
\newcommand{\out}{\circ}
\newcommand{\frob}{\mathsf{F}}
\newcommand{\proj}{\Pi}

\renewcommand{\vec}{\overline{\mathrm{vec}}}
\newcommand{\rank}{\mathrm{rank}}

\newcommand{\matxxx}{\mathrm{mat}}
\newcommand{\mat}[2]{\matxxx_{#1}(#2)}
\newcommand{\lrmat}[2]{\matxxx_{#1}\lr{#2}}
\newcommand{\Lrmat}[2]{\matxxx_{#1}\Lr{#2}}
\newcommand{\LRmat}[2]{\matxxx_{#1}\LR{#2}}

\newcommand{\figref}[1]{{Fig.~\ref{#1}}}

\pgfplotsset{
  every axis plot/.append style={line width=1pt}
}

\makeatletter
\newcommand{\pushright}[1]{\ifmeasuring@#1\else\omit\hfill$\displaystyle#1$\fi\ignorespaces}
\newcommand{\pushleft}[1]{\ifmeasuring@#1\else\omit$\displaystyle#1$\hfill\fi\ignorespaces}
\makeatother

\def\NAME{Shampoo\xspace}

\newcommand{\arxivv}[1]{#1}
\newcommand{\submitv}[1]{}

\begin{abstract}%
Preconditioned gradient methods are among the most general and powerful tools
in optimization. However, preconditioning requires storing and manipulating
prohibitively large matrices. We describe and analyze a new structure-aware
preconditioning algorithm, called \NAME, for stochastic optimization over
tensor spaces. \NAME maintains a set of preconditioning matrices, each of
which operates on a single dimension, contracting over the remaining
dimensions. We establish convergence guarantees in the stochastic convex
setting, the proof of which builds upon matrix trace inequalities.  Our
experiments with state-of-the-art deep learning models show that \NAME is
capable of converging considerably faster than commonly used optimizers.
Although it involves a more complex update rule, \NAME's runtime per step is
comparable to that of simple gradient methods such as SGD, AdaGrad, and Adam.%
\end{abstract}

\section{Introduction}

Over the last decade, stochastic first-order optimization methods have emerged
as the canonical tools for training large-scale machine learning models.
These methods are particularly appealing due to their wide applicability and
their low runtime and memory costs.

A potentially more powerful family of algorithms consists of
\emph{preconditioned} gradient methods.  Preconditioning methods maintain a matrix,
termed a preconditioner, which is used to transform (i.e., premultiply) the
gradient vector before it is used to take a step.  Classic algorithms in this
family include Newton's method, which employs the local Hessian as a
preconditioner, as well as a plethora of quasi-Newton methods (e.g.,
\cite{fletcher2013practical,lewis2013nonsmooth,nocedal1980updating}) that
can be used whenever second-order information is unavailable or too expensive
to compute.  Newer additions to this family are preconditioned online
algorithms, most notably AdaGrad~\citep{duchi2011adaptive}, that use the
covariance matrix of the accumulated gradients to form a preconditioner.

While preconditioned methods often lead to improved convergence properties,
the dimensionality of typical problems in machine learning prohibits
out-of-the-box use of full-matrix preconditioning.  To mitigate this issue,
specialized variants have been devised in which the full preconditioner is replaced
with a diagonal approximation \citep{duchi2011adaptive, kingma2014adam}, a
sketched version \citep{gonen2015faster, pilanci2017newton}, or various
estimations thereof \citep{erdogdu2015convergence, agarwal2016second,
xu2016sub}. While the diagonal methods are heavily used in practice thanks to
their favorable scaling with the dimension, the other approaches are seldom
practical at large scale as one typically requires a fine approximation (or estimate) of the preconditioner that often demands super-linear memory and computation.

In this paper, we take an alternative approach to preconditioning and describe
an efficient and practical apparatus that exploits the structure of the
parameter space.  Our approach is motivated by the observation that in
numerous machine learning applications, the parameter space entertains a more
complex structure than a monolithic vector in Euclidean space. In multiclass
problems the parameters form a matrix of size $m\times n$ where $m$ is the
number of features and $n$ is the number of classes. In neural networks, the
parameters of each fully-connected layer form an $m \times n$ matrix with $n$
being the number of input nodes and $m$ is the number of outputs. The space of
parameters of convolutional neural networks for images is a collection of $4$
dimensional tensors of the form input-depth $\times$ width $\times$ height
$\times$ output-depth. As a matter of fact, machine learning software tools
such as Torch and TensorFlow are designed with tensor structure in mind.



\input{tensors}
\begin{figure}
  \begin{center}
    \begin{tikzpicture}[scale=0.4]
      \input{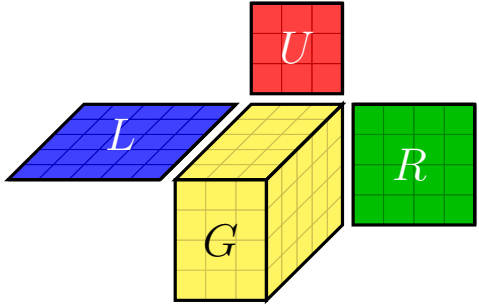}
      \fill (-1.0,-0.5) circle (0.001) node {\LARGE $G$};
      \fill (-4.3, 3.0) circle (0.001) node {\LARGE\color{white} $L$};
      \fill ( 1.5, 5.9) circle (0.001) node {\LARGE\color{white} $U$};
      \fill ( 5.2, 2.0) circle (0.001) node {\LARGE\color{white} $R$};
    \end{tikzpicture}
  \end{center}
  \caption{\label{prect:fig}Illustration of \NAME for a $3$-dimensional tensor
  $G\in\reals^{3 \times 4 \times 5}$.}
\end{figure}

Our algorithm, which we call \NAME,%
\footnote{We call it \NAME because it has to do with preconditioning.}
retains the tensor
structure of the gradient and maintains a separate preconditioner matrix for
each of its dimensions. An illustration of \NAME is provided
in~\Cref{prect:fig}. The set of preconditioners is updated by the algorithm
in an online fashion with the second-order statistics of the accumulated
gradients, similarly to AdaGrad.  Importantly, however, each individual
preconditioner is a full, yet moderately-sized, matrix that can be effectively
manipulated in large scale learning problems. 

While our algorithm is motivated by modern machine learning practices, in
particular training of deep neural networks, its derivation stems from our
analysis in a stochastic convex optimization setting.  In fact, we analyze
\NAME in the broader framework of online convex
optimization~\cite{shalev2012online, hazan2016introduction}, thus its
convergence applies more generally. Our analysis combines well-studied tools
in online optimization along with off-the-beaten-path inequalities concerning
geometric means of matrices. Moreover, the adaptation to the high-order tensor
case is non-trivial and relies on extensions of matrix analysis to the tensor
world.

We implemented \NAME (in its general tensor form) in Python as a new optimizer
in the TensorFlow framework~\citep{tensorflow}.  \NAME is extremely simple to
implement, as most of the computations it performs boil down to standard
tensor operations supported out-of-the-box in TensorFlow and similar
libraries. Using the \NAME optimizer is also a straightforward process.
Whereas recent optimization methods, such as \cite{martens2015optimizing,
neyshabur2015path}, need to be aware of the structure of the underlying model,
\NAME only needs to be informed of the tensors involved and their sizes.  In
our experiments with state-of-the-art deep learning models \NAME is capable of
converging considerably faster than commonly used optimizers.  Surprisingly,
albeit using more complex update rule, \NAME's runtime per step is comparable
to that of simple methods such as vanilla SGD.

\begin{algorithm}[t!]
\wrapalgo[0.65\textwidth]{
\begin{algorithmic}\itemindent=-10pt
\STATE Initialize $W_1 = \bm{0}_{m \times n} ~;~ L_0 = \epsilon I_m ~;~ R_0 = \epsilon I_n$
\FOR{$t = 1,\ldots,T$}
  \STATE Receive loss function $f_t:\reals^{m\times n}\mapsto\reals$
  \STATE Compute gradient $G_t = \nabla f_t(W_t)$
  \COMMENT{$G_t \in \reals^{m\times n}$}
  \STATE Update preconditioners:
  \vspace*{-2ex}
  \begin{align*}
    L_t &= L_{t-1} + G_t G_t\tr \\
    R_t &= R_{t-1} + G_t\tr G_t
  \end{align*}
  \vspace*{-4ex}
  \STATE Update parameters:
  \vspace*{-1ex}
  $$
    W_{t+1} ~=~ W_t - \eta L_t^{-\nicefrac{1}{4}} G_t R_t^{-\nicefrac{1}{4}}
  $$
  \vspace*{-4ex}
\ENDFOR
\end{algorithmic}
}
\caption{\NAME, matrix case.}
\label{alg:sham2d}
\end{algorithm}

\subsection{\NAME for matrices}

In order to further motivate our approach we start with a special case of
\NAME and defer a formal exposition of the general algorithm to later
sections.  In the two dimensional case, the parameters form a matrix
$W\in\reals^{m\times n}$.
First-order methods update iterates $W_t$ based on the gradient $G_t=\nabla
f_t(W_t)$, which is also an $m \times n$ matrix. Here, $f_t$ is the
loss function encountered on iteration $t$ that typically represents the loss
incurred over a single data point (or more generally, over a batch of data).

A structure-oblivious full-matrix preconditioning scheme
would flatten the parameter space into an $mn$-dimensional vector
and employ preconditioning matrices $H_t$ of size $mn\times mn$. 
In contrast, \NAME 
maintains smaller {left} $L_t\in\reals^{m\times m}$ and right
$R_t\in\reals^{n\times n}$ matrices containing second-moment information of the accumulated
gradients. 
On each iteration, two preconditioning matrices are formed from
$L_t$ and $R_t$ and multiply the gradient matrix from the left and right
respectively. 
The amount of space \NAME uses in the matrix case is
$m^2 + n^2$ instead of $m^2 n^2$. Moreover, as the preconditioning involves
matrix inversion (and often spectral decomposition), the amount of computation required
to construct the left and right preconditioners is $O(m^3 + n^3)$,
substantially lower than full-matrix methods which require $O(m^3 n^3)$.

The pseudocode of \NAME for the matrix case is given in~\cref{alg:sham2d}.
To recap more formally, \NAME maintains two different matrices: an $m
\times m$ matrix $L_t^{\squar}$ to precondition the rows of $G_t$ and
$R_t^\squar$ for its columns. The $\squar$ exponent arises from our
analysis; intuitively, it is a sensible choice as it induces an overall
step-size decay rate of $O(1/\sqrt{t})$, which is common in stochastic
optimization methods. 
The motivation for the algorithm comes from the
observation that its update rule is equivalent, after flattening $W_t$ and
$G_t$, to a gradient step preconditioned using the Kronecker product of
$L_t^{\squar}$ and $R_t^{\squar}$.
The latter is shown to be tightly connected to a full unstructured preconditioner matrix used by algorithms such as AdaGrad.
Thus, the algorithm can be thought of
as maintaining a ``structured'' matrix which is implicitly
used to precondition the flattened gradient, without either forming a full
matrix or explicitly performing a product with the flattened
gradient vector.

\subsection{Related work}

As noted above, \NAME is closely related to AdaGrad \citep{duchi2011adaptive}.
The diagonal (i.e., element-wise) version of AdaGrad is extremely
popular in practice and frequently applied to tasks ranging from learning
linear models over sparse features to training of large deep-learning models.
In contrast, the full-matrix version of AdaGrad analyzed in
\cite{duchi2011adaptive} is rarely used in practice due to the prohibitive
memory and runtime requirements associated with maintaining a full
preconditioner. \NAME can be viewed as an efficient, practical and provable
apparatus for approximately and implicitly using the full AdaGrad
preconditioner, without falling back to diagonal matrices.

Another recent optimization method that uses factored preconditioning is K-FAC
\citep{martens2015optimizing}, which was specifically designed to optimize the
parameters of neural networks. K-FAC employs a preconditioning scheme that
approximates the Fisher-information matrix of a generative model represented by a neural network.  The Fisher matrix of each layer in the network is approximated by a Kronecker product of two smaller matrices, relying on certain independence assumptions regarding the statistics of the gradients.
K-FAC differs from \NAME in several important ways. 
While K-FAC is used for training generative models and needs to sample from the model's predictive distribution, \NAME applies in a general stochastic (more generally, online) optimization setting and comes with convergence guarantees in the convex case.
K-FAC relies heavily on the structure of the backpropagated gradients in a feed-forward neural network. 
In contrast, \NAME is virtually oblivious to the particular model structures and only depends on standard gradient information.
As a result, \NAME is also much easier to implement and use in practice as it need not be tailored to the particular model or architecture.

\section{Background and technical tools} \label{prelim:sec}
We use lowercase letters to denote scalars and vectors and uppercase letters
to denote matrices and tensors. 
Throughout, the notation $A \succeq 0$ (resp.~$A \succ 0$) for a matrix $A$ means that $A$ is \emph{symmetric} and positive semidefinite (resp.~definite), or PSD (resp.~PD) in short.
Similarly, the notations $A \succeq B$ and $A \succ B$ mean that $A-B \succeq 0$ and $A-B \succ 0$ respectively, and both tacitly assume that $A$ and $B$ are symmetric.
Given $A \succeq 0$ and $\alpha \in \reals$, the matrix $A^\alpha$ is defined as the
PSD matrix obtained by applying $x \mapsto x^\alpha$ to the eigenvalues of $A$; formally, if we rewrite $A$ using its spectral decomposition
$\sum_{i} \lambda_i u_i^{} u_i\tr$ 
in which $(\lambda_i,u_i)$ is $A$'s $i$'th eigenpair, then
$A^\alpha = \smash{\sum_{i} \lambda_i^\alpha u_i^{} u_i\tr}$.
We denote by $\norm{x}_A = \sqrt{x\tr A x}$
the Mahalanobis norm of $x\in\reals^d$ as induced by a positive definite
matrix $A \succ 0$. The dual norm of $\norm{\cdot}_A$ is denoted
$\smash{\norm{\cdot}_{A}^*}$ and equals $\sqrt{x\tr A^{-1} x}$. The inner
product of two matrices $A$ and
$B$ is denoted as $A \bullet B = \trace(A\tr B)$.  The spectral norm of a
matrix $A$ is denoted $\norm{A}_2 = \max_{x \ne 0} \norm{Ax}/\norm{x}$ and
the Frobenius norm is $\norm{A}_\frob = \sqrt{A \bullet A}$. We denote by
$e_i$ the unit vector with $1$ in its $i$'th position and $0$ elsewhere.

\subsection{Online convex optimization}
We use Online Convex Optimization~(OCO)~\citep{shalev2012online,
hazan2016introduction} as our analysis framework. OCO can be seen as a
generalization of stochastic (convex) optimization.  In OCO a learner makes
predictions in the form of a vector belonging to a convex domain $\cW
\subseteq \reals^d$ for $T$ rounds. After predicting $w_t \in \cW$ on round
$t$, a convex function $f_t : \cW \mapsto \reals$ is chosen, potentially in an
adversarial or adaptive way based on the learner's past predictions. The
learner then suffers a loss $f_t(w_t)$ and observes the function $f_t$ as
feedback.  The goal of the learner is to achieve low cumulative loss compared
to any fixed vector in the $\cW$.  Formally, the learner attempts to minimize
its \emph{regret}, defined as the quantity
\begin{align*}
  \cR_T = \sum_{t=1}^T f_t(w_t) - \min_{w \in \cW} \sum_{t=1}^T f_t(w) ~ ,
\end{align*}
Online convex optimization includes stochastic convex optimization as a
special case. Any regret minimizing algorithm can be converted to a
stochastic optimization algorithm with convergence rate $O(\cR_T/T)$ using
an online-to-batch conversion technique~\cite{cesa2004generalization}.

\subsection{Adaptive regularization in online optimization}
We next introduce tools from online optimization that our algorithms
rely upon. First, we describe an adaptive version of Online Mirror
Descent (OMD) in the OCO setting which employs time-dependent
regularization. The algorithm proceeds as follows: on each
round $t=1,2,\ldots,T$, it receives the loss function $f_t$ and computes the
gradient $g_t = \nabla f_t(w_t)$. Then, given a positive definite
matrix $H_t \succ 0$ it performs an update according to
\begin{align} \label{eq:update}
w_{t+1} &=
\argmin_{w \,\in\, \cW} \, \Lrset{ \eta g_t \tr w
	+ \thalf\norm{w - w_t}_{H_t}^2 } ~ .
\end{align}
When $\cW = \reals^d$, \cref{eq:update} is equivalent to a preconditioned gradient step, $w_{t+1} = w_t - \eta H_t^{-1} g_t .$
More generally, the update rule can be rewritten as a projected gradient step,
\begin{align*}
w_{t+1} = \Pi_{\cW} \Lrbra{ w_t - \eta H_t^{-1} g_t ; H_t } ,
\end{align*}
where $\Pi_{\cW}[z ; H] = \argmin_{w \in \cW} \norm{w-z}_{H}$ is the
projection onto the convex set $\cW$ with respect to the norm
$\norm{\cdot}_{H}$.  
The following lemma provides a regret bound for Online
Mirror Descent, see for instance~\cite{duchi2011adaptive}.

\begin{lemma} \label{lem:regret-md}
For any sequence of matrices $H_1,\ldots,H_T \succ 0$, the regret of online
mirror descent is bounded above by,
$$
	\frac{1}{2\eta} \sum_{t=1}^T \Lr{ \norm{w_t-w^\st}_{H_t}^2 - \norm{w_{t+1}-w^\st}_{H_t}^2 }
	+
	\frac{\eta}{2} \sum_{t=1}^T \Lr{\norm{g_t}_{H_t}^*}^2 ~ .
$$
\end{lemma}

In order to analyze particular regularization schemes, namely specific
strategies for choosing the matrices $H_1,\ldots,H_T$, we need the following
lemma, adopted from~\cite{gupta2017unified}; for completeness, we provide a short proof in \cref{sec:moreproofs}. 

\begin{lemma}[\citet{gupta2017unified}] \label{lem:adareg}
Let $g_1,\ldots,g_T$ be a sequence of vectors, and let
$M_t = \sum_{s=1}^t g_s g_s\tr$ for $t \geq 1$.
Given a function $\Phi$ over PSD matrices, define
\begin{align*}
H_t = \argmin_{H \succ 0} \, \Lrset{ M_t \bullet H^{-1} + \Phi(H) }
\end{align*}
(and assume that a minimum is attained for all $t$).
Then
\begin{align*}
	\sum_{t=1}^T \Lr{\norm{g_t}_{H_t}^*}^2 \le
		\sum_{t=1}^T \Lr{\norm{g_t}_{H_T}^*}^2 + \Phi(H_T) - \Phi(H_0) ~ .
\end{align*}
\end{lemma}

\subsection{Kronecker products}
\label{sec:kron}

We recall the definition of the Kronecker product, the vectorization operation and their calculus.
Let $A$ be an $m \times n$ matrix and $B$ be an $m' \times n'$ matrix.
The Kronecker product, denoted $A \kron B$, is an $mm' \times nn'$ block
matrix defined as,
\begin{align*}
A \kron B =
\begin{pmatrix}
    a_{11} B	& a_{12} B 	& \dots 	& a_{1n} B \\
    a_{21} B	& a_{22} B 	& \dots 	& a_{2n} B \\
    \vdots		& \vdots	& \ddots	& \vdots   \\
    a_{m1} B	& a_{m2} B 	& \dots 	& a_{mn} B \\
\end{pmatrix}
~ .
\end{align*}
For an $m \times n$ matrix $A$ with rows $a_1,\ldots,a_m$, the \emph{vectorization} (or flattening)
of $A$ is the $mn \times 1$ column vector%
\footnote{This definition is slightly non-standard and differs from the more typical column-major operator $\mathrm{vec}()$; the notation $\vec()$
is used to distinguish it from the latter.}
$$\vec(A) = (\begin{matrix} a_1 & a_2 & \cdots & a_m \end{matrix})\tr.$$
The next lemma collects several properties of the Kronecker product and the $\vec(\cdot)$ operator, that will be used throughout the paper.
For proofs and further details, we refer to \cite{horn1991topics}.

\begin{lemma} \label{lem:kron-props}
Let $A,A',B,B'$ be matrices of appropriate dimensions.
The following properties hold:
\begin{enumerate}[label=(\roman*)]
\item \label{it:kron-prod}
$(A \kron B)(A' \kron B') = (A A') \kron (B B')$;
\item
$(A \kron B)\tr = A\tr \kron B\tr$;
\item \label{it:kron-pow}
If $A,B \succeq 0$, then for any $s \in \reals$ it holds that $(A \kron B)^s = A^s \kron B^s$, and in particular, if $A,B \succ 0$ then $(A \kron B)^{-1} = A^{-1} \kron B^{-1}$;
\item \label{it:kron-monotone}
If $A \succeq A'$ and $B \succeq B'$ then $A \kron B \succeq A' \kron B'$,
and in particular, if $A,B \succeq 0$ then $A \kron B \succeq 0$;
\item 
$\trace(A \kron B) = \trace(A) \trace(B)$;
\item \label{it:vec-kron}
$\vec(uv\tr) = u \kron v$ for any two column vectors $u,v$.
\end{enumerate}
\end{lemma}

The following identity connects the Kronecker product and the $\vec$
operator. It facilitates an efficient computation of a
matrix-vector product where the matrix is a Kronecker product
of two smaller matrices. We provide its proof for completeness; see \cref{sec:moreproofs}.

\begin{lemma} \label{lem:vec-kron}
Let $G \in \reals^{m\times n}$, $L \in \reals^{m\times m}$ and $R\in \reals^{n\times n}$. 
Then, one has
\begin{align*}
	(L \kron R\tr)\vec(G) = \vec(L G R) ~ .
\end{align*}
\end{lemma}

\subsection{Matrix inequalities}

Our analysis requires the following result concerning the
geometric means of matrices.
Recall that by writing $X \succeq 0$ we mean, in particular, that $X$ is a symmetric matrix.

\begin{lemma}[\citet{ando2004geometric}] \label{lem:geomean}
Assume that $0 \preceq X_i \preceq Y_i$ for all $i=1,\ldots,n$.
Assume further that all $X_i$ commute with each other and all $Y_i$
commute with each other. Let $\alpha_1,\ldots,\alpha_n \ge 0$
such that $\sum_{i=1}^n \alpha_i = 1$, then
\begin{align*}
X_1^{\alpha_1} \cdots X_n^{\alpha_n} \preceq
Y_1^{\alpha_1} \cdots Y_n^{\alpha_n} ~ .
\end{align*}
In words, the (weighted) geometric mean of commuting PSD matrices
is operator monotone.
\end{lemma}

\citet{ando2004geometric} proved a stronger result which does not require the
PSD matrices to commute with each other, relying on a generalized notion of
geometric mean, but for our purposes the simpler commuting case suffices.  We
also use the following classic result from matrix theory, attributed to
\citet{lowner1934monotone}, which is an immediate consequence of~\cref{lem:geomean}.

\begin{lemma} \label{lem:monotone}
The function $x \mapsto x^\alpha$ is operator-monotone for
$\alpha \in [0,1]$, that is, if $0 \preceq X \preceq Y$
then $X^\alpha \preceq Y^\alpha$.
\end{lemma}

\section{Analysis of \NAME for matrices}
In this section we analyze \NAME in the matrix case. The analysis
conveys the core ideas while avoiding numerous the technical details
imposed by the general tensor case. The main result of this section
is stated in the following theorem.

\begin{theorem} \label{thm:regret-2d}
Assume that the gradients $G_1,\ldots,G_T$ are matrices of
rank at most $r$.
Then the regret of \cref{alg:sham2d} compared to any
$W^\st \in \reals^{m \times n}$ is bounded as follows,
\begin{align*}
\sum_{t=1}^T f_t(W_t) - \sum_{t=1}^T f_t(W^\st)
\le
\sqrt{2r}D \trace(L_T^{\squar}) \trace(R_T^{\squar})
~,
\end{align*}
where
$$L_T = \eps I_m + \sum_{t=1}^T G_t G_t\tr \;,\;\;
  R_T = \eps I_n + \sum_{t=0}^T G_t\tr G_t \;,\;\;
  D = \max_{t \in [T]} \norm{W_t-W^\st}_\frob ~.$$
\end{theorem}

Let us make a few comments regarding the bound.  First, under mild
conditions, each of the trace terms on the right-hand side of the bound
scales as $O(T^{1/4})$. Thus, the overall scaling of the bound with respect to the number of iterations $T$ is $O(\sqrt{T})$,
which is the best possible in the context of online (or stochastic)
optimization.
For example, assume that the functions $f_t$ are $1$-Lipschitz with respect
to the spectral norm, that is, $\norm{G_t}_2 \le 1$ for all~$t$. Let us also
fix $\eps=0$ for simplicity. Then, $G_t G_t\tr \preceq I_m$ and $G_t\tr G_t
\preceq I_n$ for all $t$, and so we have $\trace(L_T^{\squar}) \le
mT^{\squar}$ and $\trace(R_T^{\squar}) \le nT^{\squar}$.  That is, in the
worst case, while only assuming convex and Lipschitz losses, the regret of
the algorithm is $\O(\sqrt{T})$.

Second, we note that $D$ in the above bound could in principle grow with the
number of iterations $T$ and is not necessarily bounded by a constant. This
issue can be easily addressed, for instance, by adding an additional step to the algorithm
in which $W_t$ is projected $W_t$ onto the convex set of matrices whose
Frobenius norm is bounded by $D/2$. 
Concretely, the projection at step $t$ needs to be computed with respect to
the norm induced by the pair of matrices $(L_t,R_t)$, defined as
$\norm{A}_t^2 = \trace(A\tr L_t^{\squar}AR_t^{\squar})$;
it is not hard to verify that the latter indeed defines a norm over
$\reals^{m \times n}$, for any $L_t,R_t \succ 0$. Alas, the projection becomes
computationally expensive in large scale problems and is rarely performed in
practice. We therefore omitted the projection step from \cref{alg:sham2d} in
favor of a slightly looser bound.

The main step in the proof of the theorem is established in the following lemma. The lemma
implies that the Kronecker product of the two preconditioners used by the
algorithm is lower bounded by a full $mn \times mn$ matrix often
employed in full-matrix preconditioning methods.

\begin{lemma} \label{thm:moo-lower}
Assume that $G_1,\ldots,G_T \in \reals^{m \times n}$ are matrices of
rank at most $r$. Let $g_t = \vec(G_t)$ denote the vectorization of $G_t$ for
all $t$.  Then, for any $\eps \ge 0$,
\begin{align*}
\eps I_{mn} + \frac{1}{r} \sum_{t=1}^T g_t^{} g_t\tr
\preceq
	\LR{\eps I_{m} +
	\sum_{t=1}^T G_t^{} G_t\tr}^{\shalf} \kron \LR{\eps I_{n} +
	\sum_{t=1}^T G_t\tr G_t^{}}^{\shalf} ~ .
\end{align*}
\end{lemma}
In particular, the lemma shows that the small eigenvalues of the full-matrix preconditioner
on the left, which are the most important for effective preconditioning, do
not vanish as a result of the implicit approximation.
In order to prove \cref{thm:moo-lower} we need the following technical
result.
\begin{lemma} \label{lem:kron-base}
Let $G$ be an $m \times n$ matrix of rank at most $r$ and denote $g =
\vec(G)$. Then,
\begin{align*}
\frac{1}{r} gg\tr
\preceq
I_{m} \kron (G\tr G)
\quad\text{and}\quad
\frac{1}{r} gg\tr
\preceq
(GG\tr) \kron I_{n}
~ .
\end{align*}
\end{lemma}
\begin{proof}
Write the singular value decomposition $G = \sum_{i=1}^r \sigma_i^{} u_i^{} v_i\tr$,
where $\sigma_i \ge 0$ for all $i$, and $u_1,\ldots,u_r \in \reals^{m}$ and
$v_1,\ldots,v_r \in \reals^{n}$ are orthonormal sets of vectors.
Then,
$g = \sum_{i=1}^r \sigma_i (u_i \kron v_i)$
and hence,
\begin{align*}
gg\tr
&=
\LR{\sum_{i=1}^r \sigma_i (u_i \kron v_i)}\LR{\sum_{i=1}^r \sigma_i (u_i \kron v_i)}\tr
.
\end{align*}
Next, we use the fact that for any set of vectors $w_1,\ldots,w_r$,
$$ \LR{\sum_{i=1}^r w_i}\LR{\sum_{i=1}^r w_i}\tr \preceq r \sum_{i=1}^r w_i^{} w_i\tr ~,$$
which holds since given a vector $x$ we can write $\alpha_i = x\tr w_i$, and
use the convexity of $\alpha \mapsto \alpha^2$ to obtain
$$
x\tr \LR{\sum_{i=1}^r w_i}\LR{\sum_{i=1}^r w_i}\tr x
=
\LR{\sum_{i=1}^r \alpha_i}^2
\le
r \sum_{i=1}^r \alpha_i^2
=
r \, x\tr \LR{\sum_{i=1}^r w_i^{} w_i\tr} x ~
.
$$
Using this fact and \cref{lem:kron-props}\ref{it:kron-prod} we can rewrite,
\begin{align*}
gg\tr
&=
\LR{\sum_{i=1}^r \sigma_i (u_i \kron v_i)}\LR{\sum_{i=1}^r \sigma_i (u_i \kron v_i)}\tr \\
&\preceq
r \sum_{i=1}^r \sigma_i^2 (u_i \kron v_i)(u_i \kron v_i)\tr
\\
&=
r \sum_{i=1}^r \sigma_i^2 (u_i^{} u_i\tr) \kron (v_i^{} v_i\tr)
~.
\end{align*}
Now, since
$GG\tr = \sum_{i=1}^r \sigma_i^2 u_i^{} u_i\tr$
and $v_i^{} v_i\tr \preceq I_{n}$ for all $i$, we have
\begin{align*}
\frac{1}{r} gg\tr
\preceq
\sum_{i=1}^r \sigma_i^2 (u_i^{} u_i\tr) \kron I_{n}
=
(GG\tr) \kron I_{n} ~ .
\end{align*}
Similarly, using $G\tr G = \sum_{i=1}^r \sigma_i^2 v_i^{} v_i\tr$ and
$u_i^{} u_i\tr \preceq I_{m}$ for all $i$, we
obtain the second matrix inequality.
\end{proof}

\begin{proof}[Proof of \cref{thm:moo-lower}] 
Let us introduce the following notations to simplify our
derivation,
\begin{align*}
	A_m \eqdef \eps I_{m} + \sum_{t=1}^T G_t^{} G_t\tr & ~ , \qquad
	B_n \eqdef \eps I_{n} + \sum_{t=1}^T G_t\tr G_t^{} ~.
\end{align*}
From \cref{lem:kron-base} we know that,
\begin{align*}
\eps I_{mn} + \frac{1}{r} \sum_{t=1}^T g_t^{} g_t\tr
	&\preceq I_{m} \kron B_n ~~\mbox{ and }~~
\eps I_{mn} + \frac{1}{r} \sum_{t=1}^T g_t^{} g_t\tr
	\preceq A_m \kron I_{n} ~.
\end{align*}
Now, observe that $I_{m} \kron B_n$ and $A_m \kron I_{n}$ commute with each other.
Using \cref{lem:geomean} followed by \cref{lem:kron-props}\ref{it:kron-pow} and \cref{lem:kron-props}\ref{it:kron-prod} yields
\begin{align*}
\eps I_{mn} + \frac{1}{r} \sum_{t=1}^T g_t^{} g_t\tr
&\preceq
	\Lr{ I_{m} \kron B_n}^{\shalf}
	\Lr{ A_m \kron I_{n}}^{\shalf}
\,=\,
	\Lr{ I_{m} \kron B_n^{\shalf}}
	\Lr{ A_m^{\shalf} \kron I_{n}}
\,=\,
A_m^{\shalf} \kron B_n^{\shalf} ~ ,
\end{align*}
which completes the proof.
\end{proof}

We can now prove the main result of the section.

\begin{proof}[Proof of \cref{thm:regret-2d}]
Recall the update performed in~\cref{alg:sham2d},
\begin{align*}
	W_{t+1}^{} ~=~ W_t^{} - \eta L_t^{-\nicefrac{1}{4}} G_t^{} R_t^{-\nicefrac{1}{4}} ~.
\end{align*}
Note that the pair of left and right preconditioning matrices, $\smash{L_t^{\squar}}$ and $\smash{R_t^{\squar}}$, is equivalent due to \cref{lem:vec-kron} to a
single preconditioning matrix $H_t = \smash{L_t^{\squar} \kron R_t^{\squar}}
\in \reals^{mn\times mn}$.  This matrix is applied to flattened version of the
gradient $g_t = \vec(G_t)$. More formally, letting $w_t = \vec(W_t)$ we have
that the update rule of the algorithm is equivalent to,
\begin{align}
	w_{t+1} = w_t - \eta H_t^{-1} g_t^{} ~ . \label{eqn:flat_update}
\end{align}
Hence, we can invoke \cref{lem:regret-md} in conjuction the fact that
$0 \prec H_1 \preceq \ldots \preceq H_T$. The latter follows from \cref{lem:kron-props}\ref{it:kron-monotone}, as $0 \prec L_1 \preceq
\ldots \preceq L_T$ and $0 \prec R_1 \preceq \ldots \preceq R_T$. We thus
further bound the first term of~\cref{lem:regret-md} by,
\begin{align} \label{eq:telescope}
\sum_{t=1}^T (w_t-w^\st)\tr (H_t - H_{t-1}) (w_t-w^\st)
	\le D^2 \sum_{t=1}^T \trace(H_t - H_{t-1}) = D^2 \trace(H_T) ~.
\end{align}
for
$D = \max_{t \in [T]} \norm{w_t-w^\st} =
	\max_{t \in [T]} \norm{W_t-W^\st}_\frob$
where $w^\st = \vec(W^\st)$ and $H_0 = 0$.  We obtain the regret bound
\begin{align} \label{eq:2d-ineq0}
\sum_{t=1}^T f_t(W_t) - \sum_{t=1}^T f_t(W^\st)
\le
\frac{D^2}{2\eta} \trace(H_T)
+ \frac{\eta}{2} \sum_{t=1}^T \Lr{\norm{g_t}_{H_t}^*}^2 ~ .
\end{align}
Let us next bound the sum on the right-hand side of \cref{eq:2d-ineq0}.
First, according to \cref{thm:moo-lower} and the monotonicity (in the operator sense) of the square root function $x \mapsto x^{1/2}$ (recall \cref{lem:monotone}), for the preconditioner $H_t$ we have that
\begin{align} \label{eq:2d-ineq1}
	\wh{H}_t \eqdef \LR{ r \eps I + \sum_{s=1}^t g_s^{} g_s\tr }^{\shalf} 
	\preceq
	\sqrt{r} H_t
	~ .
\end{align}
On the other hand, invoking \cref{lem:adareg} with the choice of potential
$$\Phi(H) = \trace(H) + r\eps\trace(H^{-1})$$ and
$M_t = \sum_{s=1}^t g_t^{} g_t\tr$, we get,
\begin{align*}
\argmin_{H \succ 0} \Lrset{ M_t \bullet H^{-1} + \Phi(H) }
	= \argmin_{H \succ 0} \trace\Lr{ \wh{H}_t^2 H^{-1} + H }
	= \wh{H}_t ~.
\end{align*}
To see the last equality, observe that for any symmetric $A \succeq 0$, the function $\trace(AX+X^{-1})$ is minimized at $X = A^{-\shalf}$, since $\nabla_X \trace(AX+X^{-1}) = A-X^{-2}$.
Hence, \cref{lem:adareg} implies
\begin{align} \label{eq:2d-ineq2}
\sum_{t=1}^T \Lr{\norm{g_t}_{\widehat{H}_t}^*}^2
&\le
\sum_{t=1}^T \Lr{\norm{g_t}_{\widehat{H}_T}^*}^2 + \Phi(\wh{H}_T) - \Phi(\wh{H}_0)
\notag\\
&\le
\LR{ r \eps I + \sum_{t=1}^T g_t g_t\tr } \bullet \wh{H}_T^{-1} + \trace(\wh{H}_T)
\\
&=
2\trace(\wh{H}_T) ~ .  \notag
\end{align}
Using \cref{eq:2d-ineq1} twice along with \cref{eq:2d-ineq2}, we obtain
\begin{align*}
\sum_{t=1}^T (\norm{g_t}_{H_t}^*)^2
\le
\sqrt{r} \sum_{t=1}^T (\norm{g_t}_{\widehat{H}_t}^*)^2
\le
2\sqrt{r} \trace(\wh{H}_T)
\le
2r \trace(H_T) ~ .
\end{align*}
Finally, using the above upper bound in \cref{eq:2d-ineq0} and choosing $\eta = D/\sqrt{2r}$ gives the
desired regret bound:
\begin{align*}
\sum_{t=1}^T f_t(W_t) - \sum_{t=1}^T f_t(W^\st)
\le
\LR{\frac{D^2}{2\eta} + \eta r} \trace(H_T)
=
\sqrt{2r} D \trace(L_T^\squar) \trace(R_T^\squar) ~ .
\end{align*}
\end{proof}

\section{\NAME for tensors}

In this section we introduce the \NAME algorithm in its general form, which
is applicable to tensors of arbitrary dimension.  Before we can
present the algorithm, we review further definitions and operations involving
tensors.

\subsection{Tensors: notation and definitions}

A tensor is a multidimensional array.  The \emph{order} of a tensor is the
number of dimensions (also called modes).  For an order-$k$ tensor $A$ of
dimension $n_1 \times \cdots \times n_k$, we use the notation
$A_{j_1,\ldots,j_k}$ to refer to the single element at position $j_i$ on the
$i$'th dimension for all $i$ where $1 \le j_i \le n_i$.
We also denote
$$n = \prod_{i=1}^k n_i ~~ \mbox{ and } ~~ \forall i: \, n_{-i} = \prod_{j \ne i} n_j ~ .$$
The following definitions are used throughout the section.
\begin{itemize}
\item
A \emph{slice} of an order-$k$ tensor along its $i$'th dimension is a tensor of order $k-1$
which consists of entries with the same index on the $i$'th dimension. A slice generalizes the
notion of rows and columns of a matrix.

\item
An $n_1 \times \cdots \times n_k$ tensor $A$ is of \emph{rank one} if it can be written as an outer product of
$k$ vectors of appropriate dimensions. Formally, let $\out$ denote the vector outer product and
and set $A = u^1 \out u^2 \out \cdots \out u^k$ where $u^i \in \reals^{n_i}$ for all $i$. Then $A$ is an
order-$k$ tensor defined through
\begin{align*}
A_{j_1, \ldots, j_k}
&=
(u^1 \out u^2 \out \cdots \out u^k)_{j_1, \ldots, j_k}
\\
&=
u^1_{j_1} u^2_{j_2} \cdots u^k_{j_k}
,
\qquad
\forall ~
1 \le j_i \le n_i ~ (i \in [k])~
.
\end{align*}

\item
The \emph{vectorization} operator flattens a tensor to a column vector in $\reals^{n}$, generalizing the matrix $\vec$ operator.
For an $n_1 \times \cdots \times n_k$ tensor $A$ with slices $A^1_1,\ldots,A^1_{n_1}$ along its first dimension, this operation can be defined recursively as follows:
\begin{align*}
\vec(A) = \Lr{ \begin{matrix} \vec(A^1_1)\tr & \cdots & \vec(A^1_{n_1})\tr \end{matrix} }\tr
,
\end{align*}
where for the base case ($k=1$), we define $\vec(u) = u$ for any column vector $u$.

\item
The \emph{matricization} operator $\mat{i}{A}$ reshapes a tensor $A$ to a matrix by vectorizing the slices of $A$ along the $i$'th dimension and stacking them as rows of a matrix.
More formally, for an $n_1 \times \cdots \times n_k$ tensor $A$ with slices $A^i_1,\ldots,A^i_{n_i}$ along the $i$'th dimension, matricization is defined as the $n_i \times n_{-i}$ matrix,
\begin{align*}
\mat{i}{A}
=
\Lr{ \begin{matrix} \vec(A^i_1) & \cdots & \vec(A^i_{n_i}) \end{matrix} }\tr ~
.
\end{align*}

\item
The matrix product of an $n_1 \times \cdots \times n_k$ tensor $A$ with an $m \times n_i$ matrix $M$ is defined as the $n_1 \times \cdots \times n_{i-1} \times m \times n_{i+1} \times
\cdots \times n_k$ tensor, denoted $A\times_i M$, for which the identity
$\mat{i}{A\times_i M} = M \mat{i}{A}$ holds.
Explicitly, we define $A\times_i M$ element-wise as
\begin{align*}
(A \times_i M)_{j_1,\ldots,j_k}
=
\sum_{s=1}^{n_i} M_{j_i s} A_{j_1, \ldots j_{i-1}, s, j_{i+1}, \ldots, j_k}.
\end{align*}
A useful fact, that follows directly from this definition, is that the tensor-matrix product is commutative, in the sense that $A \times_i M \times_{i'} M' = A \times_{i'} M' \times_{i} M$ for any $i \neq i'$ and matrices $M \in \reals^{n_i \times n_i}$, $M' \in \reals^{n_{i'} \times n_{i'}}$.

\item
The \emph{contraction} of an $n_1 \times \cdots \times n_k$ tensor $A$ with itself along all but the $i$'th dimension is an $n_i
\times n_i$ matrix defined as $A^{(i)} = \mat{i}{A} \mat{i}{A}\tr$, or
more explicitly as
\begin{align*}
A^{(i)}_{j,j'}
=
\sum_{\alpha_{-i}} A_{j,\alpha_{-i}} A_{j',\alpha_{-i}}
\qquad
\forall ~ 1 \le j,j' \le n_i
,
\end{align*}
where the sum ranges over all possible indexings $\alpha_{-i}$ of all dimensions $\ne i$.
\end{itemize}

\subsection{The algorithm}

We can now describe the \NAME algorithm in the general, order-$k$ tensor
case, using the definitions established above.  Here we assume that the
optimization domain is $\cW = \reals^{n_1 \times \cdots \times n_k}$, that
is, the vector space of order-$k$ tensors, and the functions
$f_1,\ldots,f_T$ are convex over this domain.  In particular, the gradient
$\nabla f_t$ is also an $n_1 \times \cdots \times n_k$ tensor.

The \NAME algorithm in its general form, presented in \cref{alg:kd}, is
analogous to \cref{alg:sham2d}. It maintains a separate preconditioning
matrix $H_t^i$ (of size $n_i \times n_i$) corresponding to for each
dimension $i \in [k]$ of the gradient.  On step $t$, the $i$'th mode of
the gradient $G_t$ is then multiplied by the matrix
$(H^{i}_t)^{-\nicefrac{1}{2k}}$ through the tensor-matrix product operator
$\times_i$.  (Recall that the order in which the multiplications are carried
out does not affect the end result and can be arbitrary.) After all
dimensions have been processed and the preconditioned gradient $\wt{G}_t$
has been obtained, a gradient step is taken.

\begin{algorithm}[t]
\wrapalgo[0.65\textwidth]{
  \begin{algorithmic}\itemindent=-10pt
  \STATE Initialize: $W_1 = \bm{0}_{n_1\times\cdots\times n_k}$ ;~
   $\forall i\in[k]: \, H^{i}_0 = \epsilon I_{n_i}$
    \FOR{$t = 1,\ldots,T$}
    \STATE Receive loss function $f_t : \reals^{n_1\times\cdots\times n_k} \mapsto \reals$
    \STATE Compute gradient $G_t = \nabla f_t(W_t)$
    	\COMMENT{$G_t \in \reals^{n_1\times\cdots\times n_k}$}
    \STATE $\wt{G}_t \gets G_t$
    	\COMMENT{$\wt{G}_t$ is preconditioned gradient}
    \FOR{$i = 1,\ldots,k$}
    \STATE $H^{i}_t = H^{i}_{t-1} + G_t^{(i)}$
    \STATE $\wt{G}_t \gets \wt{G}_t \times_i (H^{i}_t)^{-\nicefrac{1}{2k}}$
    \ENDFOR
    \STATE Update: $W_{t+1} = W_t - \eta \wt{G}_t$
    \ENDFOR
  \end{algorithmic}
}
\caption{\NAME, general tensor case.}
\label{alg:kd}
\end{algorithm}

The tensor operations $A^{(i)}$ and $M \times_i A$ can be implemented using tensor contraction,
which is a standard library function in scientific computing libraries such as Python's NumPy, and is fully supported by modern machine learning frameworks such as TensorFlow \citep{tensorflow}.
See \cref{sec:impl} for further details on our implementation of the algorithm in the TensorFlow environment.

We now state the main result of this section.

\begin{theorem} \label{thm:sham-kd}
Assume that for all $i \in [k]$ and $t=1,\ldots,T$ it holds that $\rank(\mat{i}{G_t}) \le r_i$, and let $r = (\prod_{i=1}^{k} r_i)^{\nicefrac{1}{k}}$.
Then the regret of \cref{alg:kd} compared to any $W^\st \in \reals^{n_1 \times \cdots \times n_k}$ is
\begin{align*}
\sum_{t=1}^T f_t(W_t) - \sum_{t=1}^T f_t(W^\st)
\le
\sqrt{2r}D \prod_{i=1}^k \trace\Lr{(H_T^{i})^{\nicefrac{1}{2k}}}
,
\end{align*}
where $H_T^i = \eps I_{n_i} + \sum_{t=1}^T G_t^{(i)}$ for all $i \in [k]$ and
 $D = \max_{t \in [T]} \norm{W_t-W^\st}_\frob$.
\end{theorem}

The comments following \cref{thm:regret-2d} regarding the parameter $D$ in the above bound and the lack of projections in the algorithm are also applicable in the general tensor version.
Furthermore, as in the matrix case, under standard assumptions each of the trace terms on the right-hand side of the above bound is bounded by $\O(T^{\nicefrac{1}{2k}})$. Therefore, their product, and thereby the overall regret bound, is $\O(\sqrt{T})$.

\subsection{Analysis}

We turn to proving \cref{thm:sham-kd}.
For the proof, we require the following generalizations of \cref{thm:moo-lower,lem:vec-kron} to tensors of arbitrary order.

\begin{lemma} \label{thm:moo-tensor}
Assume that $G_1,\ldots,G_T$ are all order $k$ tensors of dimension $n_1 \times \cdots \times n_k$, and let $n = n_1 \cdots n_k$ and $g_t = \vec(G_t)$ for all $t$.
Let $r_i$ denote the bound on the rank of the $i^{\footnotesize \mbox{th}}$ matricization of  $G_1,\ldots,G_T$, namely,
$\rank(\mat{i}{G_t}) \le r_i$ for all $t$ and $i \in [k]$. Denote $r = (\prod_{i=1}^{k} r_i)^{\nicefrac{1}{k}}$.
Then, for any $\eps \ge 0$ it holds that
\begin{align*}
\eps I_{n} + \sum_{t=1}^T g_t g_t\tr
~\preceq~
r \, \bigkron_{i=1}^k \LR{\eps I_{n_i} + \sum_{t=1}^T G_t^{(i)}}^{\nicefrac{1}{k}} ~ .
\end{align*}
\end{lemma}

\begin{lemma} \label{lem:vec-kron-tensor}
Let $G$ be an $n_1\times\ldots\times n_k$ dimensional tensor and $M_i$ be an $n_i \times n_i$ for $i \in [k]$ , then
\begin{align*}
\LR{\bigkron_{i=1}^k M_i}\vec(G)
=
\vec(G \times_1 M_1 \times_2 M_2 \ldots \times_k M_k) ~
.
\end{align*}
\end{lemma}

We defer proofs to \cref{app:tensor}. The proof of our main theorem now readily follows.

\begin{proof}[Proof of \cref{thm:sham-kd}]
The proof is analogous to that of \cref{thm:regret-2d}.
For all $t$, let 
$$H_t = (\bigkron_{i=1}^k H_t^i)^{\nicefrac{1}{2k}} ~~, ~~ g_t = \vec(G_t) ~~ , ~~ w_t = \vec(W_t) ~.$$
Similarly to the order-two (matrix) case, and in light of \cref{lem:vec-kron-tensor}, the update rule of the algorithm is equivalent to
$
w_{t+1} = w_t - \eta H_t^{-1} g_t.
$
The rest of the proof is identical to that of the matrix case, using \cref{thm:moo-tensor} in place of \cref{thm:moo-lower}.
\end{proof}

\section{Implementation details}\label{sec:impl}

We implemented \NAME in its general tensor form in Python as a new
TensorFlow~\citep{tensorflow} optimizer. Our implementation follows almost
verbatim the pseudocode shown in \cref{alg:kd}.  We used the built-in
\texttt{tensordot} operation to implement tensor contractions and
tensor-matrix products. Matrix powers were computed simply by constructing a singular value decomposition (SVD) and then taking the powers of the singular values.  These operations are fully supported in TensorFlow. 
We plan to implement \NAME in PyTorch in the near future.

Our optimizer treats each tensor in the input model as a separate
optimization variable and applies the \NAME update to each of these tensors
independently. This has the advantage of making the optimizer entirely
oblivious to the specifics of the architecture, and it only has to be aware
of the tensors involved and their dimensions.  In terms of
preconditioning, this approach amounts to employing a block-diagonal
preconditioner, with blocks corresponding to the different tensors in the
model. In particular, only intra-tensor correlations are captured and
correlations between parameters in different tensors are ignored entirely.

Our optimizer also implements a diagonal variant of \NAME which is
automatically activated for a dimension of a tensor whenever it is
considered too large for the associated preconditioner to be stored in
memory or to compute its SVD. Other dimensions of the same tensor are not
affected and can still use non-diagonal preconditioning (unless they are too
large themselves). 
See \cref{sec:diagonal} for a detailed description of this variant and its analysis.
In our experiments, we used a threshold of around 1200 for each dimension to trigger the diagonal version with no apparent sacrifice in performance. This
option gives the benefit of working with full preconditioners whenever
possible, while still being able to train models where some of the tensors
are prohibitively large, and without having to modify either the architecture
or the code used for training.

\section{Experimental results}
\label{sec:experiments}

We performed experiments with \NAME on several datasets, using standard deep neural-network models.
We focused on two domains: image classification on CIFAR-10/100, and statistical language modeling on LM1B.
In each experiment, we relied on existing code for training the models, and merely replaced the TensorFlow optimizer without making any other changes to the code.

In all of our experiments, we worked with a mini-batch of size 128.
In \NAME, this simply means that the gradient $G_t$ used in each iteration of the algorithm is the average of the gradient over 128 examples, but otherwise has no effect on the algorithm.
Notice that, in particular, the preconditioners are also updated once per batch using the averaged gradient rather than with gradients over individual examples.

We made two minor heuristic adjustments to \NAME to improve performance.
First, we employed a delayed update for the preconditioners, and recomputed the roots of the matrices $H_t^i$ once in every 20--100 steps.
This had almost no impact on accuracy, but helped to improve the amortized runtime per step.
Second, we incorporated momentum into the gradient step, essentially computing the running average of the gradients $\overline{G}_t = \alpha \overline{G}_{t-1} + (1-\alpha)G_t$ with a fixed setting of $\alpha=0.9$.
This slightly improved the convergence of the algorithm, as is the case with many other first-order stochastic methods.

Quite surprisingly, while the \NAME algorithm performs significantly more computation per step than algorithms like SGD,
AdaGrad, and Adam, its actual runtime in practice is not much worse.
\cref{tab:runtimes} shows the average number of steps (i.e., batches of size 128) per second on a Tesla K40 GPU, for each of the algorithms we tested.
As can be seen from the results, each step of \NAME is typically slower than that of the other algorithms by a small margin, and in some cases (ResNet-55) it is actually faster.

\begin{table}
\centering
\setlength{\tabcolsep}{3pt}
\begin{tabular}{|l|c|c|c|c|}
\hline
Dataset & SGD & Adam & AdaGrad & \NAME\\
\hline
CIFAR10 (ResNet-32) & 2.184 & 2.184 & 2.197 & 2.151\\
CIFAR10 (Inception) &  3.638 & 3.667 & 3.682 & 3.506\\
CIFAR100 (ResNet-55) & 1.210 & 1.203 & 1.210 & 1.249\\
LM1B (Attention) & 4.919 & 4.871 & 4.908 & 3.509\\
\hline
\end{tabular}
\caption{Average number of steps per second (with batch size of 128) in each experiment, for each of the algorithms we tested.}
\label{tab:runtimes}
\end{table}

\subsection{Image Classification}


\begin{figure}[t]
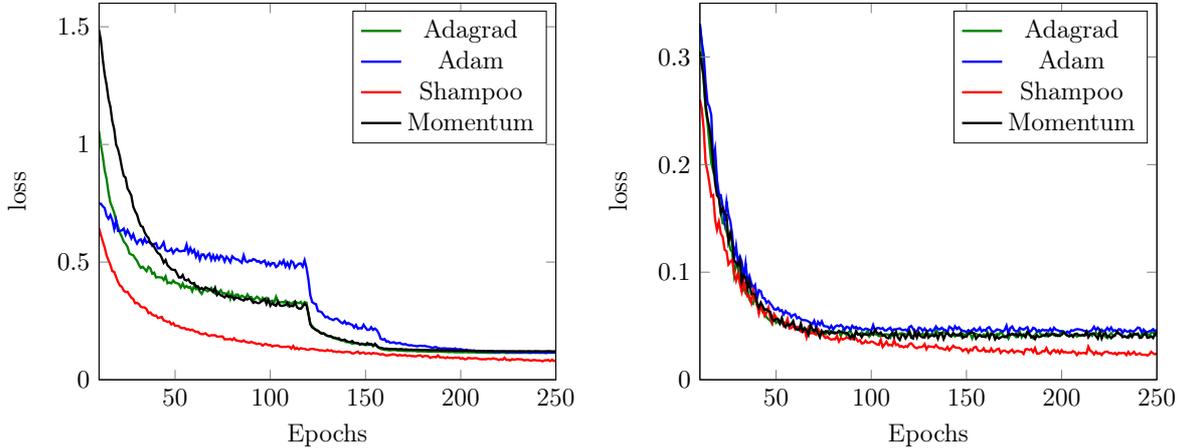

\centering
\resizebox{1.0\textwidth}{!}{
\input{train_TFResnet_loss}
\input{train_BatchNorm_loss}
}
\caption{Training loss for a residual network and an inception network on CIFAR-10.}
\label{fig:cifar10}
\end{figure}


We ran the CIFAR-10 benchmark with several different architectures. 
For each optimization algorithm, we explored 10 different learning rates between 0.01 and 10.0 (scaling the 
entire range for Adam), and chose the one with the best loss and error. 
We show in \cref{fig:cifar10} the training loss for a 32-layer residual network with 2.4M parameters. This
network is capable of reaching an error rate of 5\% on the test set. 
We also ran on the 20-layer small inception network described in \citet{zhang2016understanding},
with 1.65M trainable parameters, capable of reaching an error rate of 7.5\% on test data.

For CIFAR-100 (\cref{fig:cifar100}), we used a 55-layer residual network with 13.5M trainable parameters. In this model, the trainable variables are all
tensors of order $4$ (all layers are convolutional), where the largest layer is of dimension $(256, 3, 3, 256)$.
This architecture does not employ batch-norm, dropout, etc., and was able to reach an error rate of 24\% on the test set.

\begin{figure}[t]
\centering
\resizebox{0.5\textwidth}{!}{
\input{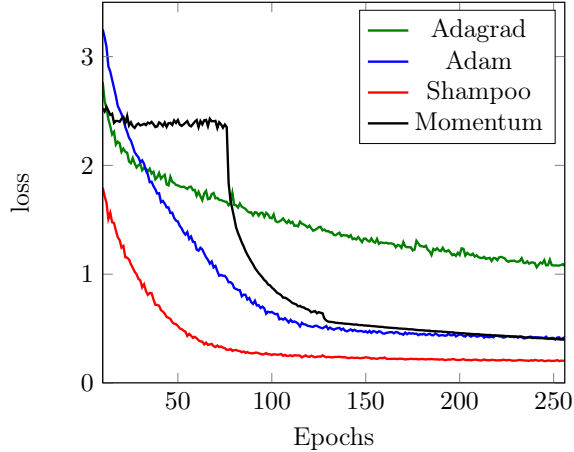}
}
\caption{Training loss for a residual network on CIFAR-100 (without batchnorm).}
\label{fig:cifar100}
\end{figure}

\subsection{Language Models}

Our next experiment was on the LM1B benchmark for statistical language modeling \citep{lm1b}.
We used an Attention model with 9.8M trainable parameters from \cite{vaswani2017attention}. This model has a succession of fully connected-layers, 
with corresponding tensors of order at most $2$, the largest of which is of dimension $(2000, 256)$.
In this experiment, we simply used the default learning rate of $\eta=1.0$ for \NAME. For the other algorithms we explored various different settings of the learning rate.
The graph for the test perplexity is shown in \cref{fig:t2t-neg-log-perplexity}.

\begin{figure}[t!]
\centering
\resizebox{0.5\textwidth}{!}{
\input{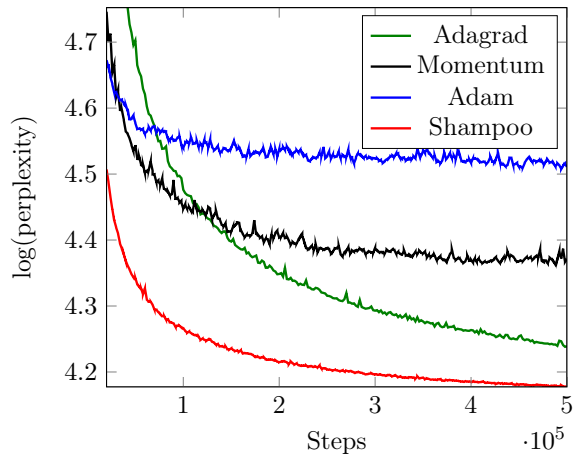}
}
\caption{Test log-perplexity of an Attention model of \citet{vaswani2017attention}.}
\label{fig:t2t-neg-log-perplexity}
\end{figure}

\subsection*{Acknowledgements}

We are grateful to Samy Bengio, Roy Frostig, Phil Long, Aleksander M\k{a}dry and Kunal Talwar for numerous discussions and helpful suggestions.
Special thanks go to Roy Frostig for coming up with the name ``Shampoo.''

\bibliographystyle{abbrvnat}
\bibliography{bib}

\appendix

\section{Diagonal \NAME}
\label{sec:diagonal}
In this section we describe a diagonal version of the \NAME algorithm, in which each of the preconditioning matrices is a diagonal matrix. 
This diagonal variant is particularly useful if one of the dimensions is too large to store the corresponding full preconditioner in memory and to compute powers thereof. 
For simplicity, we describe this variant in the matrix case.
The only change in \cref{alg:sham2d} is replacing the updates of the matrices $L_t$ and $R_t$ with the updates
\begin{align*}
L_t &= L_{t-1} + \diag(G_t^{} G_t\tr) ;\\
R_t &= R_{t-1} + \diag(G_t\tr G_t^{}) .
\end{align*}
Here, $\diag(A)$ is defined as $\diag(A)_{ij} = \ind{i=j} A_{ij}$ for all $i,j$.
See \cref{alg:sham2d-diag} for the resulting pseudocode.
Notice that for implementing the algorithm, one merely needs to store the diagonal elements of the matrices $L_t$ and $R_t$ and maintain $O(m+n)$ numbers is memory.
Each update step could then be implemented in $O(mn)$ time, i.e., in time linear in the number of parameters.

\begin{algorithm}[t!]
\wrapalgo[0.65\textwidth]{
\begin{algorithmic}\itemindent=-10pt
\STATE Initialize $W_1 = \bm{0}_{m \times n} ~;~ L_0 = \epsilon I_m ~;~ R_0 = \epsilon I_n$
\FOR{$t = 1,\ldots,T$}
  \STATE Receive loss function $f_t:\reals^{m\times n}\mapsto\reals$
  \STATE Compute gradient $G_t = \nabla f_t(W_t)$
  \COMMENT{$G_t \in \reals^{m\times n}$}
  \STATE Update preconditioners:
  \vspace*{-2ex}
	\begin{align*}
	L_t &= L_{t-1} + \diag(G_t^{} G_t\tr)\\
	R_t &= R_{t-1} + \diag(G_t\tr G_t^{})
	\end{align*}
  \vspace*{-4ex}
  \STATE Update parameters:
  \vspace*{-1ex}
  $$
    W_{t+1} ~=~ W_t - \eta L_t^{-\nicefrac{1}{4}} G_t^{} R_t^{-\nicefrac{1}{4}}
  $$
  \vspace*{-4ex}
\ENDFOR
\end{algorithmic}
}
\caption{Diagonal version of \NAME, matrix case.}
\label{alg:sham2d-diag}
\end{algorithm}

We note that one may choose to use the full \NAME update for one dimension while employing the diagonal version for the other dimension. (In the more general tensor case, this choice can be made independently for each of the dimensions.)
Focusing for now on the scheme described in \cref{alg:sham2d-diag}, in which both dimensions use a diagonal preconditioner, we can prove the following regret bound.

\begin{theorem} \label{thm:regret-2d-diag}
Assume that the gradients $G_1,\ldots,G_T$ are matrices of rank at most $r$.
Then the regret of \cref{alg:sham2d-diag} compared to any $W^\st \in \reals^{m \times n}$ is bounded as
\begin{align*}
\sum_{t=1}^T f_t(W_t) - \sum_{t=1}^T f_t(W^\st)
\le
\sqrt{2r}D_\infty \trace(L_T^{\squar}) \trace(R_T^{\squar})
~,
\end{align*}
where
$$L_T = \eps I_m + \sum_{t=1}^T \diag(G_t G_t\tr) \;,\;\;
  R_T = \eps I_n + \sum_{t=0}^T \diag(G_t\tr G_t) \;,\;\;
  D_\infty = \max_{t \in [T]} \norm{W_t-W^\st}_\infty ~.$$
(Here, $\norm{A}_\infty$ is the entry-wise $\ell_\infty$ norm of a matrix $A$, i.e., $\norm{A}_\infty \eqdef \norm{\vec(A)}_\infty$.)
\end{theorem}

\begin{proof}[Proof (sketch).]
For all $t$, denote $H_t = L_t^{\squar} \kron R_t^{\squar}$.
The proof is identical to that of \cref{thm:regret-2d}, with two changes.
First, we can replace \cref{eq:telescope} with
\begin{align*}
\sum_{t=1}^T (w_t-w^\st)\tr (H_t - H_{t-1}) (w_t-w^\st)
	\le D_\infty^2 \sum_{t=1}^T \trace(H_t - H_{t-1}) = D_\infty^2 \trace(H_T) ~,
\end{align*}
where the inequality follows from the fact that for a diagonal PSD matrix $M$ one has $v\tr M v \le \norm{v}_\infty^2 \trace(M)$.
Second, using the facts that $A \preceq B \Rightarrow \diag(A) \preceq \diag(B)$ and $\diag(A \kron B) = \diag(A) \kron \diag(B)$, we can show (from \cref{thm:moo-lower,lem:monotone}) that
\begin{align*}
\wh{H}_t \eqdef 
\diag\LR{r\eps I_{mn} + \sum_{s=1}^t g_t^{} g_t\tr}^{\shalf}
\preceq
\sqrt{r} L_t^{\squar} \kron R_t^{\squar} 
=
\sqrt{r} H_t
~,
\end{align*}
replacing \cref{eq:2d-ineq1}.
Now, proceeding exactly as in the proof of \cref{thm:regret-2d} with \cref{eq:telescope,eq:2d-ineq1} replaced by the above facts leads to the result.
\end{proof}

\section{Tensor case: Technical proofs}
\label{app:tensor}

We prove \cref{thm:moo-tensor,lem:vec-kron-tensor}.  We require several
identities involving the $\vec(\cdot)$ and $\mat{i}{\cdot}$ operations,
bundled in the following lemma.

\begin{lemma} \label{lem:vec-mat-tensor}
For any column vectors $u^1,\ldots,u^k$ and order-$k$ tensor $A$ it holds that:
\begin{enumerate}[label=(\roman*)]
\item
$\vec(u^1 \out \cdots \out u^k) = u^1 \kron \cdots \kron u^k$~ ;
\item
$\mat{i}{u^1 \out \cdots \out u^k} = u^i \big(\bigkron_{i'\ne i} u^{i'} \big)\tr$ ~ ;
\item
$\vec(A) = \vec(\mat{1}{A}) = \vec(\mat{k}{A}\tr)$ ~ ;
\item
$\mat{i}{A\times_i M} = M \mat{i}{A}$~ .
\end{enumerate}
\end{lemma}

\begin{proof}
\begin{enumerate}[nosep,label=(\roman*),leftmargin=0pt,itemindent=4ex]
\item
The statement is trivially true for $k=1$.
The $j$'th slice of $u^1 \out \cdots \out u^k$ along the first dimension is $u^1_j(u^2 \out \cdots \out u^k)$.
By induction, we have
\begin{align*}
&\vec(u^1 \out \cdots \out u^k)\\
&= (u^1_1(u^2 \kron \cdots \kron u^k)\tr, \cdots, u^1_{n_1}(u^2 \kron \cdots \kron u^k)\tr)\tr\\
&= u^1 \kron \cdots \kron u^k
.
\end{align*}
\item
The $j$'th slice of $u^1 \out \cdots \out u^k$ along the $i$-th dimension is $u^i_j(u^1 \out \cdots u^{i-1} \out u^{i+1}\out \cdots \out u^k)$. Thus
\begin{align*}
\mat{i}{u^1 \out \cdots \out u^k}
= \LR{u^i_1\bigkron_{j\ne i} u^j, \ldots, u^i_{n_i}\bigkron_{j\ne i} u^j}\tr
= u^i \LR{\bigkron_{j\ne i} u^j}\tr
.
\end{align*}
\item
If $A$ is a rank one tensor $u^1 \out \cdots \out u^k$, then we have
\begin{align*}
\vec(\mat{1}{A}) &= \vec(u^1 (u^2 \kron \cdots \kron u^k)\tr)
\\
&=
u^1 \kron \cdots \kron u^k
=
\vec(A) ~
,
\end{align*}
and
\begin{align*}
\vec(\mat{k}{A}\tr)
&=
\vec((u^k (u^1 \kron \cdots \kron u^{k-1})\tr)\tr)
\\
&=
\vec((u^1 \kron \cdots \kron u^{k-1})(u^k)\tr)
\\
&=
u^1 \kron \cdots \kron u^k
=
\vec(A) ~
.
\end{align*}
As any tensor can be written as a sum of rank-one tensors, the identity extends to arbitrary tensors due to the linearity of $\mat{i}{\cdot}$ and $\vec(\cdot)$.
\item
If $A = u^1 \out \cdots \out u^k$ is a rank one tensor, then from the definition it follows that
\begin{align*}
A \times_i M = u^1 \out \cdots\out u^{i-1} \out Mu^i \out u^{i+1} \out \cdots \out u^k ~ .
\end{align*}
Therefore,  from (ii) above, we have
\begin{align*}
\mat{i}{A \times_i M} 
= Mu^i  \LR{\bigkron_{j\ne i} u^j}\tr
= M \mat{i}{A} ~
.
\end{align*}
As above, this property can be extended to an arbitrary tensor $A$ due to the linearity of all operators involved.\qedhere
\end{enumerate}
\end{proof}

\subsection{Proof of \cref{thm:moo-tensor}}

We need the following technical result.

\begin{lemma}\label{lem:transpose-tensor}
Let $G$ be an order $k$ tensor of dimension $n_1 \times \cdots \times n_k$, and $B$ an
$n_i\times n_i$ matrix.  Let $g_i = \vec(\mat{i}{G})$ and $g = \vec(G)$. Then
\begin{align*}
g_i^{} g_i\tr
\preceq
B \kron \LR{ \bigkron_{j\neq i} I_{n_j} }
~~~\Leftrightarrow~~~
gg\tr
\preceq
\LR{ \bigkron_{j<i} I_{n_j} } \kron B \kron \LR{ \bigkron_{j>i} I_{n_j} }
.
\end{align*}
\end{lemma}
\begin{proof}
Let $X$ be any $n_1 \times \cdots \times n_k$ dimensional tensor, and denote
$$x = \vec(X) ~ , ~ x_i = \vec(\mat{i}{X}) ~. $$
We will show that
\begin{align*}
x_i\tr  g_i^{}g_i\tr x_i^{}
\le
x_i\tr \LRbra{ B \kron \LR{ \bigkron_{j\neq i} I_{n_j} } } x_i
~~~\Leftrightarrow~~~
x\tr g g\tr x
\le
x\tr \LRbra{\LR{ \bigkron_{j<i} I_{n_j} } \kron B \kron \LR{ \bigkron_{j>i} I_{n_j} }} x
~,
\end{align*}
which would prove the lemma.
We first note that the left-hand sides of the inequalities are equal, as both are equal to the square of the dot-product of the tensors $G$ and $X$ (which can be defined as the dot product of $\vec(G)$ and $\vec(X)$).
We will next show that the right-hand sides are equal as well.

Let us write $X = \sum_\alpha X_\alpha (e_{\alpha_1} \out \cdots \out e_{\alpha_k})$, where
 $\alpha$ ranges over all
$k$-tuples such that $\alpha_j \in [n_j]$ for $j \in [k]$, and $e_{\alpha_j}$ is an $n_j$-dimensional
unit vector with 1 in the $\alpha_j$ position, and zero elsewhere.
Now, $x = \vec(X) = \sum_\alpha X_\alpha (e_{\alpha_1} \kron \cdots \kron e_{\alpha_k})$.
Thus
\begin{align*}
x\tr \LRbra{\LR{ \bigkron_{j<i} I_{n_j} } \kron B \kron \LR{ \bigkron_{j>i} I_{n_j} }} x
&=
x\tr \sum_\alpha X_\alpha \LRbra{\LR{ \bigkron_{j<i} I_{n_j} } \kron B \kron \LR{ \bigkron_{j>i} I_{n_j} }} (e_{\alpha_1} \kron \cdots \kron e_{\alpha_k})
\\
&=
x\tr \sum_\alpha X_\alpha \LRbra{\LR{ \bigkron_{j<i} e_{\alpha_j} } \kron B e_{\alpha_i} \kron \LR{ \bigkron_{j>i} e_{\alpha_j} }}
\\
&=
\sum_{\alpha, \alpha'} X_\alpha X_{\alpha'}
\LRbra{\LR{ \bigkron_{j<i} e_{\smash{\alpha'_j}}\tr e_{\smash{\alpha_j^{}}}^{} } \kron e_{\smash{\alpha'_i}}\tr B e^{}_{\alpha_i^{}} \kron \LR{ \bigkron_{j>i} e_{\smash{\alpha'_j}}\tr e^{}_{\alpha_j^{}} }}
\\
&=
\sum_{\alpha, \alpha'_i} B_{\alpha'_i\alpha_i} X_{\alpha_1\ldots \alpha_i\ldots \alpha_k}  X_{\alpha_1\ldots \alpha'_i\ldots \alpha_k}
,
\end{align*}
since $e_{\smash{\alpha'_j}}\tr e^{}_{\alpha_j^{}} = 1$ if $\alpha'_j = \alpha_j$, and $e_{\smash{\alpha'_j}}\tr e^{}_{\alpha_j^{}} =  0$ otherwise.

On the other hand, recall that
$$
\mat{i}{X}
=
\sum_\alpha X_\alpha e_{\alpha_i} \LR{\bigkron_{j \ne i} e_{\alpha_j}}\tr
,
$$
thus
$$
x_i
=
\vec(\mat{i}{X})
=
\sum_\alpha X_\alpha \LR{e_{\alpha_i} \kron \bigkron_{j \ne i} e_{\alpha_j}}
$$
and therefore
\begin{align*}
x_i\tr \LRbra{ B \kron \LR{ \bigkron_{j\neq i} I_{n_j} } } x_i
&=
x_i\tr \sum_\alpha X_\alpha \LRbra{ B \kron \LR{ \bigkron_{j\neq i} I_{n_j} } } \LR{e_{\alpha_i} \kron \bigkron_{j \ne i} e_{\alpha_j}}
\\
&=
x_i\tr \sum_\alpha X_\alpha \LR{ Be_{\alpha_i} \kron \bigkron_{j \ne i} e_{\alpha_j} }
\\
&=
\sum_{\alpha'} \sum_\alpha X_\alpha X_{\alpha'}
\LR{ e_{\smash{\alpha'_i}}\tr B e^{}_{\alpha_i} \kron \bigkron_{j \ne i} e_{\smash{\alpha'_j}}\tr e^{}_{\alpha_j} }
\\
&=
\sum_{\alpha, \alpha'_i} B_{\alpha'_i\alpha_i} X_{\alpha_1\ldots \alpha_i\ldots \alpha_k}  X_{\alpha_1\ldots \alpha'_i\ldots \alpha_k}
.
\end{align*}
To conclude, we have shown that
\begin{align*}
x_i\tr \LRbra{ B \kron \LR{ \bigkron_{j\neq i} I_{n_j} } } x_i
=
x\tr \LRbra{\LR{ \bigkron_{j<i} I_{n_j} } \kron B \kron \LR{ \bigkron_{j>i} I_{n_j} }} x
,
\end{align*}
and as argued above, this proves the lemma.
\end{proof}

\begin{proof}[Proof of \cref{thm:moo-tensor}]
Consider the matrix $\mat{i}{G_t}$.  By \cref{lem:kron-base}, we have
\begin{align*}
\frac{1}{r_i} \vec(\mat{i}{G_t})\vec(\mat{i}{G_t})\tr
\preceq
(\mat{i}{G_t}\mat{i}{G_t}\tr) \kron I_{n_{-i}}
=
G_t^{(i)} \kron \LR{ \bigkron_{j \ne i} I_{n_j} }
\end{align*}
(recall that $n_{-i} = \prod_{j\neq i} n_j$). 
Now, by \cref{lem:transpose-tensor}, this implies
\begin{align*}
\frac{1}{r_i} g_tg_t\tr
\preceq
\LR{\bigkron_{j=1}^{i-1} I_{n_j}} \kron G_t^{(i)} \kron \LR{\bigkron_{j=i+1}^{k} I_{n_j}}
.
\end{align*}
Summing over $t=1,\ldots,T$ and adding $\eps I_n$, we have for each dimension $i \in [k]$ that:
\begin{align*}
\eps I_{n} + \sum_{t=1}^T g_t g_t\tr
&\preceq
r_i \eps I_{n} + r_i \, \LR{\bigkron_{j=1}^{i-1} I_{n_j}} \kron \LR{\sum_{t=1}^T G_t^{(i)}} \kron \LR{\bigkron_{j=i+1}^{k} I_{n_j}}
\\
&=
r_i \, \LR{\bigkron_{j=1}^{i-1} I_{n_j}} \kron \LR{ \eps I_{n_i} + \sum_{t=1}^T G_t^{(i)} } \kron \LR{\bigkron_{j=i+1}^{k} I_{n_j}}
.
\end{align*}
The matrices on the right-hand sides of the $k$ inequalities are positive semidefinite and commute with each other, so we are in
a position to apply \cref{lem:geomean} and obtain the result.
\end{proof}

\subsection{Proof of \cref{lem:vec-kron-tensor}}

\begin{proof}
The proof is by induction on $k \ge 2$.
The base case ($k = 2$) was already proved in \cref{lem:vec-kron}.
For the induction step, let $H = \bigkron_{i=1}^{k-1} M_i$.
Using the relation $\vec(G) = \vec(\mat{k}{G}\tr)$ and then \cref{lem:vec-kron}, the left-hand side of the identity is
\begin{align*}
\LR{\bigkron_{i=1}^k M_i}\vec(G)
=
(H \kron M_k)\vec(\mat{k}{G}\tr)
=
\vec(H\mat{k}{G}\tr M_k\tr)
.
\end{align*}
Now, consider the slices $G^k_1 ,\ldots, G^k_{n_k}$ of $G$ along the $k$'th dimension (these are $n_k$ tensors of order $k-1$). 
Then the $i$'th row of $\mat{k}{G}$ is the vector $\vec(G^k_i)\tr$.  Applying
the induction hypothesis to $H \vec(G^k_i)$, we get
\begin{align*}
H\vec(G^k_i)
=
\LR{\bigkron_{i=1}^{k-1} M_i} \vec(G^k_i)
=
\vec(G^k_i  \times_1 M_1 \times_2 M_2 \cdots \times_{k-1} M_{k-1})
.
\end{align*}
Stacking the $n_k$ vectors on both sides (for $i = 1,\ldots,n_k$) to form
$n_k \times n_{-k}$ matrices, we get
$$
H \mat{k}{G}\tr
=
\mat{k}{G \times_1 M_1 \times_2 M_2 \cdots \times_{k-1} M_{k-1}}\tr
.
$$
Now, let $G' = G \times_1 M_1 \times_2 M_2 \cdots \times_{k-1}
M_{k-1}$.
Substituting, it follows that
\begin{align*}
\smash{\LR{\bigkron_{i=1}^k M_i}} \vec(G)
&=
\vec(\mat{k}{G'}\tr M_k\tr)
\\
&= \vec((M_k\mat{k}{G'})\tr)
\\
&=
\vec(\mat{k}{G' \times_k M_k}\tr)
&&\qquad\because ~ B\mat{i}{A} = \mat{i}{A \times_i B}
\\
&=
\vec(G' \times_k M_k)
&&\qquad\because ~ \vec(\mat{k}{G}\tr) = \vec(G)
\\
&=
\vec(G \times_1 M_1 \cdots \times_k M_k)
.&&\qedhere
\end{align*}
\end{proof}

\section{Additional proofs}
\label{sec:moreproofs}

\subsection{Proof of \cref{lem:adareg}}

\begin{proof}
The proof is an instance of the Follow-the-Leader / Be-the-Leader
(FTL-BTL) Lemma of~\citet{kalai2005efficient}.
We rewrite the inequality we wish to prove as
\begin{align*}
\sum_{t=1}^T  \Lr{\norm{g_t}_{H_t}^*}^2 + \Phi(H_0)
	\leq \sum_{t=1}^T  \Lr{\norm{g_t}_{H_T}^*}^2 + \Phi(H_T) ~ .
\end{align*}
The proof proceeds by an induction on $T$.
The base of the induction, $T=0$, is trivially true. Inductively, we have
\begin{align*}
\sum_{t=1}^{T-1}  \Lr{\norm{g_t}_{H_t}^*}^2 +
	\Phi(H_0) &\leq \sum_{t=1}^{T-1}  \Lr{\norm{g_t}_{H_{T-1}}^*}^2 +
	\Phi(H_{T-1})\\
	&\leq \sum_{t=1}^{T-1}  \Lr{\norm{g_t}_{H_{T}}^*}^2 + \Phi(H_{T}) ~.
\end{align*}
The second inequality follows from the fact that $H_{T-1}$
is a minimizer of
$$M_{T-1} \bullet H^{-1} + \Phi(H) =
	\sum_{t=1}^{T-1}  \Lr{\norm{g_t}_H^*}^2 + \Phi(H) ~ .$$
Adding $(\norm{g_t}_{H_{T}}^*)^2$ to both sides gives the result.
\end{proof}

\subsection{Proof of \cref{lem:vec-kron}}

\begin{proof}
We first prove the claim for $G$ of rank one, $G = u v\tr$.
Using first \ref{it:vec-kron} and then \ref{it:kron-prod} from \cref{lem:kron-props}, the left
hand side is,
\begin{align*}
(L \kron R\tr) \vec(G)
	&= (L \kron R\tr) \vec(u v\tr)
	 = (L \kron R\tr) (u\kron v)
	 = (Lu) \kron (R\tr v)
	~ .
\end{align*}
For the right hand side we have,
\begin{align*}
 \vec(L G R)
	&= \vec(L u v\tr R) = \vec\Lr{L u (R\tr v)\tr} = (Lu) \kron (R\tr v) ~ ,
\end{align*}
where we used \ref{it:vec-kron} from \cref{lem:kron-props} for the last equality. Thus we
proved the identity for $G = u v\tr$. More generally, any matrix can be expressed as a sum of rank one matrices, thus the identity follows from the linearity of all the operators involved.
\end{proof}

\end{document}